\documentclass{MITcsail}

\usepackage{latexsym}
\usepackage{tikz}
\usetikzlibrary{arrows,automata,fit,matrix}
\usepackage{wrapfig}
\usetikzlibrary{arrows.meta}
\colorlet{goodcolor}{blue!80!black}
\colorlet{badcolor}{orange!70!black}
\usepackage[frozencache,cachedir=.]{minted}

\title{Entangled Residual Mappings}

\author
{\small Mathias Lechner~$^{1, 2}$\footnote{Correspondence E-mail: mlechner@mit.edu},  Ramin Hasani~$^{1}$, Zahra Babaiee~$^{3}$, \small Radu Grosu~$^{3}$, Daniela Rus~$^{1}$, Thomas A. Henzinger~$^{2}$,  Sepp Hochreiter~$^{4}$\\
\vspace{0.5em}
\normalfont{\small $^{1}$Massachusetts Institute of Technology (MIT), Cambridge, MA, USA}\\
\normalfont{\small $^{2}$Institute of Science and Technology Austria (IST Austria), Klosterneuburg, Austria}\\
\normalfont{\small $^{3}$Technische Universit\"at Wien (TU Wien), Vienna, Austria}\\
\normalfont{\small $^{4}$Johannes Kepler Universität (JKU) Linz, Austria} \vspace{2em}
}

\begin{document} 



\maketitle 
\thispagestyle{firstpagestyle}

\begin{abstract}
Residual mappings have been shown to perform representation learning in the first layers and iterative feature refinement in higher layers. This interplay, combined with their stabilizing effect on the gradient norms, enables them to train very deep networks. In this paper, we take a step further and introduce entangled residual mappings to generalize the structure of the residual connections and evaluate their role in iterative learning representations. An entangled residual mapping replaces the identity skip connections with specialized entangled mappings such as orthogonal, sparse, and structural correlation matrices that share key attributes (eigenvalues, structure, and Jacobian norm) with identity mappings. We show that while entangled mappings can preserve the iterative refinement of features across various deep models, they influence the representation learning process in convolutional networks differently than attention-based models and recurrent neural networks. In general, we find that for CNNs and Vision Transformers entangled sparse mapping can help generalization while orthogonal mappings hurt performance. For recurrent networks, orthogonal residual mappings form an inductive bias for time-variant sequences, which degrades accuracy on time-invariant tasks. \vspace{1em}
\end{abstract}

\section{Introduction}
\noindent The behavior of residual mappings in neural networks has been extensively studied. The common observation is that residual blocks enable both iterative refinement of features \citep{greff2016highway} and deep representation learning \citep{jastrzebski2018residual}. These schemes are enabled by preserving the magnitude of the gradients in the reverse-mode automatic differentiation  \citep{rumelhart1986learning,hochreiter1991untersuchungen} through identity mappings. For example, to learn long-term time dependencies by a neural network, the long short-term memory (LSTM) \citep{hochreiter1997long} deployed a constant error-propagation to preserve the magnitude of the gradients inside a deep unfolded recurrent neural networks (RNNs). 

With a similar principle, residual connections made the training of very deep neural networks possible \citep{greff2016highway,he2016deep}. A residual block, $R_f(x)$, provides two pathways for error propagation by using the following mechanism: $R_f(x) := f(x,\theta)+I_n x$. Here, $f$ is a trainable neural network parameterized by $\theta$ receiving inputs $x$, and $I_n$ is a non-trainable identity matrix to map the inputs through a second path to provide a clear error-propagation channel.

\begin{figure}[t]
\centering
	\includegraphics[scale=1]{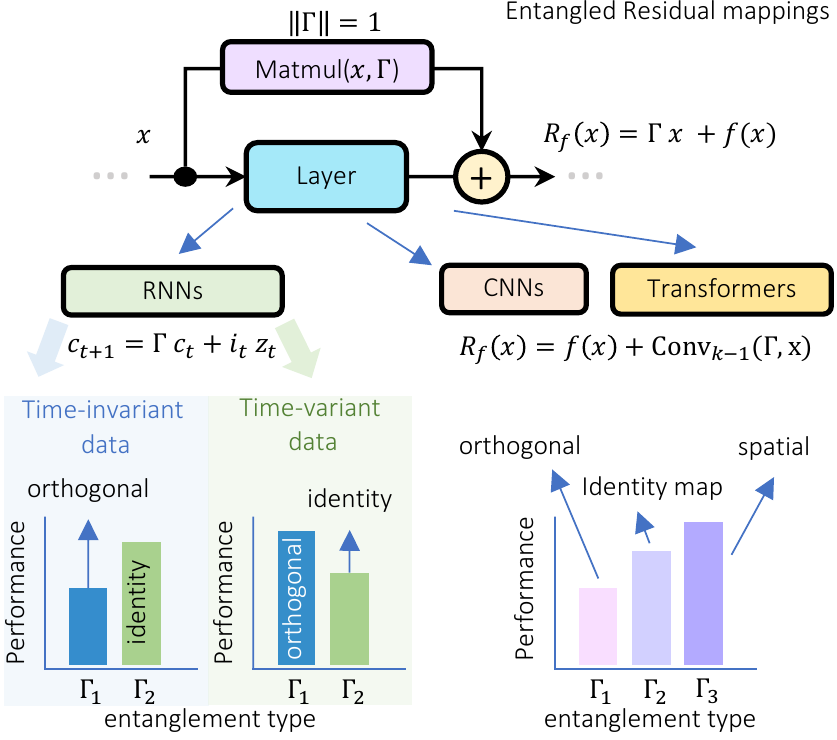}
	\caption{Different types of entangled residual mappings affect the generalizability of deep networks differently.}
	\label{fig:intro}
\end{figure}
Identity mappings in deep networks have been vastly studied, specially in vision-based classification tasks \citep{srivastava2015training,he2016deep,he2016identity,szegedy2017inception,xie2017aggregated}. Their primary purpose was to supply a neural network with a gateway for better gradient flow to facilitate learning \citep{he2016identity}. However, \citep{greff2016highway} further showed empirically that residual mapping in ResNets and Highway networks \citep{srivastava2015highway,zilly2017recurrent} in vision contexts do not learn new representations but instead tend to iteratively improve the representation under each residual block. The iterative refinement scheme further got confirmed as being a key reason for the great performance of residual networks \citep{jastrzebski2018residual,casanova2018iterative,zhang2019cascaded,guo2019dynamic}. In recurrent networks, unitary/orthogonal RNNs circumvented the problem of vanishing/exploding gradients by conditioning the hidden-to-hidden transition matrix's eigenvalues at 1, using unitary/orthogonal mappings \citep{arjovsky2016unitary,jing2017tunable,lezcano2019cheap,lechner2020learning}.

The common principle in these studies is the use of identity mappings as skip connection. In this paper, we set out to take a step further and ask whether a more general formulation of residual connections and transition matrices exist that yield the same or even better performance? Can we encode a prior on the learned features by adding structural biases to the residual/unitary mappings to improve performance while preserving the iterative feature refinement property? Are the vanilla identity mappings used in ResNets, for instance, already optimal? How do answers to these questions differ in vision-, recurrent-, and attention-based networks?

Here, we answer the questions raised above by building a parent family of identity mappings in deep networks and carefully studying their properties. This family is called \emph{entangled residual mappings} which generalizes the identity mapping $\text{id}: x \mapsto I_n x$, by replacing the identity matrix $I_n$ with specialized matrices $\Gamma$. Different types of entanglement matrices $\Gamma$ may share key attributes with identity matrices. Such attributes include conditions on eigenvalues, the map's structure and sparsity, and the Jacobian norm's magnitude. We study residual mappings' behavior at each of these fronts and report their impact on representation learning. We experiment comprehensively with various deep learning architectures ranging from advanced vision networks, RNNs, to models based on attention mechanisms such as Transformers \citep{vaswani2017attention}. Figure \ref{fig:intro} portraits our experimental setup, while Figure \ref{fig:resnets} takes an insightful look into its findings.

\begin{figure}[t]
    \centering
    \includegraphics[width=0.5\textwidth]{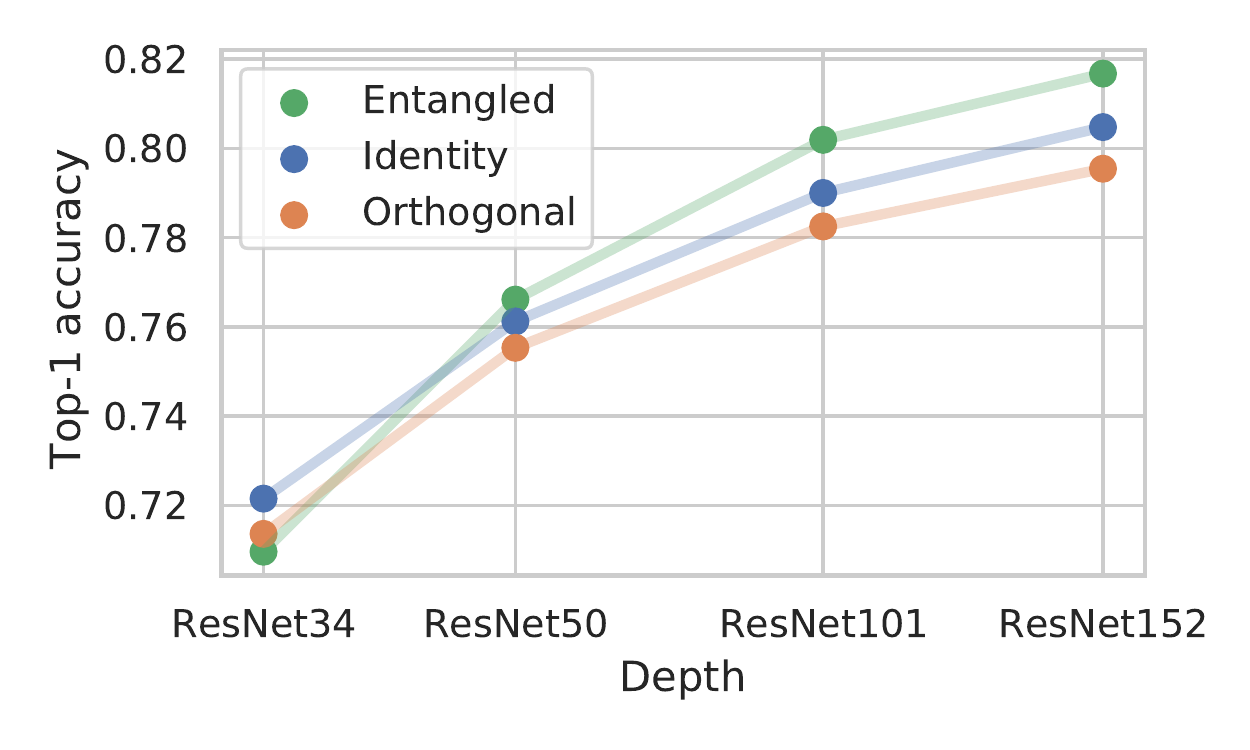}
    \caption{ImageNet top-1 accuracy of residual neural networks with various types of skip pathways. Orthogonal connections consistently reduces accuracy, while our entangled residual connections improve learning with increasing network depth.}
    \label{fig:resnets}
\end{figure}

\noindent\textbf{Summary of contributions}
\begin{itemize}
\itemsep-0.8mm
    \item Introducing entangled residual mappings as a framework to understand the role of residual mappings across various deep learning models and studying their eigenvalues, sparsity, and orthogonality, while preserving iterative inference properties under small entanglement coefficients
    \item Entangled mappings influence the representation learning process differently in convolutional networks, attention-based models, and recurrent neural networks.
    \item  For CNNs and Vision Transformers entangled sparse mappings help generalization while orthogonal mappings do not. For recurrent networks orthogonal residual mappings improve performance on time-variant tasks, and hurt performance on time-invariant tasks.
\end{itemize}

\section{Related works}
\noindent\textbf{Studies of Error Propagation}
\citep{hochreiter1991untersuchungen} and subsequently \citep{bengio1994learning} discovered that the magnitude of the error-flow in backpropagation-based learning schemes determines whether deep representations can be learned or not.
In particular, the long-short-term-memory (LSTM) \citep{hochreiter1997long} was the first architecture specifically designed to allow a near-constant magnitude error-backpropagation through its memory mechanisms. This enabled the RNNs to learn long-term dependencies in sequential data \citep{hochreiter1997long,gers2000learning,greff2016lstm,lechner2020gershgorin,hasani2019response,gu2021efficiently}.
The same principle has been adopted for learning feed-forward architectures under the term \textit{highway} networks \citep{srivastava2015training}. This mechanism made possible the training of deep networks with up to 1000 layers.
\citep{he2016deep} further utilized and simplified the idea in the form of residual neural networks (\textit{ResNets}) with skip connections to ease the flow of gradients through very deep networks. 

\noindent\textbf{Architectures with Residual connections.}
Since the development of ResNets \citep{he2016deep}, a large body of research has modified the micro-blocks' inner function in residual networks. This gave rise to a large family of models with skip connections \citep{szegedy2017inception,xie2017aggregated}.
In parallel, DenseNet \citep{huang2017densely} uses densely connected paths to interlink each micro-block with all the successive micro-blocks, to gain better degrees of expressivity with the cost of increasing complexity. 

In the context of recurrent models, \citep{soltani2016higher} introduced weighted skip connections between subsequent states of an unfolded RNN learn long-term dependencies. Inspired by this work, Dual Path Networks \citep{chen2017dual} proposed a combination of ResNets and DenseNet to enjoy the simultaneous use of feature recovery and novel feature generation of both models, for learning better representations. \citep{liao2016bridging} unifies the formulation of RNNs and ResNets to build a more performant generalization of the architectures with skip connections. 

More recent works \citep{chen2018neural,grathwohl2018ffjord,behrmann2019invertible,chen2019residual,massaroli2020dissecting} on invertible mappings and flows enabled the use of skip connections for classification, density estimation and generation. The primary focus of these works is on how to enhance the performance and expressivity of flows while reducing the computational complexity of the training pipeline, to enable scalability. In this work, in contrast, we design a framework to understand how skip connections can help or hurt generalization.

In the context of natural processing tasks, \citep{wang2022deepnet} has shown empirically that scaling the residual pathways of attention-based transformer networks is necessary for learning very deep transformer models. The authors have experimentally demonstrated that this scaling factor depends on the task and type of transformer block, i.e., encoder require a different factor than the decoder.
Instead of multiplying the residual connection by a scalar value, our works investigates structured, i.e., entangled, changes to the skip path. Consequently, our entangled connections study the effects on the learned representation instead of simply its magnitude.

\noindent\textbf{Purpose of Residual Connections}
Residual connections enabled the training of very deep neural networks \citep{he2016deep}, by preserving the gradients' magnitude during the backward pass. \citep{orhan2017skip} showed that these connections also eliminate singularities that cause difficulties during training deep models \citep{wei2008dynamics}. Moreover, skip connections were shown to make the system learn feature re-use rather than generating new representations \citep{greff2016highway,chen2017dual}. The family of continuous-depth models \citep{chen2018neural,hasani2021liquid} were proposed as the continuous analogous of ResNets, with computational advantages and realization of infinite-depth networks. However, It was recently shown that neural ordinary differential equations (ODEs) are indeed ResNets without ODE-interpretations \citep{ott2021resnet}. Amongst works on continuous models, the class of liquid networks \citep{hasani2021liquid,hasani2021closed} showed that an input-dependent variation of the identity mappings in the state-transition matrix of ODE-based networks can enhance generalization across many applications \citep{lechner2019designing,hasani2020natural,lechner2020neural,vorbach2021causal}. Here motivated by liquid networks we set out to explore the behavior of residual networks under change of their identity mappings

The work of \citep{balduzzi2017shattered} has observed that skip connections introduce correlations into the stochastic gradients. In particular, the authors show that the gradients of a standard deep network resembles white-noise which hurts the learning process. Contrarily, the gradients of residual networks have a more brown-noise characteristic. 

Other works \citep{yang2017mean,tarnowski2019dynamical} have identified that residual connections stabilize the geometry of the forward dynamics, whereas plain feed-forward networks suffer from a spectrum and geometry collapse with increasing depth.




\noindent \textbf{Unitary networks} do not require explicit memory and residual connections by constraining the architecture such that the Jacobian of each layer has its absolute eigenvalues at 1 \citep{arjovsky2016unitary}.
While this ensures a constant-magnitude error propagation, the process of enforcing the constraints can negatively impact the capacity and complexity of models \citep{wisdom2016full}.

\noindent\textbf{Orthogonal weight profiles} have been shown to overcome the vanishing/exploding gradient problem in RNNs~\citep{pmlr-v48-arjovsky16,lezcanocasado2019cheap,dorobantu2016dizzyrnn,pmlr-v70-vorontsov17a}. Orthogonal weights in CNNs improve the stability of the activations distributions in layers~\citep{2017regularizing}, derive faster and more stable convergence~\citep{NEURIPS2018_bf424cb7}, and facilitated back propagation by improving the spectrum of local Hessian~\citep{NEURIPS2018_b6d67a24}. Highlighting the importance of weight initialization and driven by the activation and gradient norm-preserving property of orthogonal weight matrices,~\citep{saxe2014exact,mishkin2016need,NEURIPS2019_e520f70a} demonstrated that orthogonal weight initialization helps in training very deep networks and improves their performances. Several works further propose orthogonal weight regularization to ensure that the weight matrix orthogonality is preserved during training. \citep{DBLP:journals/corr/HarandiF16,DBLP:journals/corr/OzayO16} proposed Stiefel manifold based constraint of weights, and~\citep{huang2017orthogonal} formulated this problem as Optimization over Multiple Dependent Stiefel Manifolds (OMDSM) to generalize the square orthogonal matrix to rectangular for feed forward networks. \citep{xie2017need} proposed orthogonal regularization by forcing the Gram matrix of the weight matrices to be close to identity under Frobenius norm. Utilizing analytical tools like mutual coherence and restricted isometry property,~\citep{bansal2018gain} demonstrated that orthogonality regularizations can achieve better accuracy and faster convergence. To ensure orthogonal convolutions, recently~\citep{Wang_2020_CVPR} proposed an approach  based on the doubly block-Toeplitz (DBT) matrix representation of the convolutional kernel called Orthogonal CNNs (OCNNs). \citep{orhan2017skip} investigated orthogonality in skip connections and showed that networks with random dense orthogonal skip connections can achieve the same performance of networks with identity skip connections.

\section{Main Results}
\noindent In this section, we construct the formal framework within which we explore the role of residual connections. We first introduce a general notion for residual mapping, namely the entangled residual mapping. We then explain how properties such as eigenvalues, sparsity, and orthogonality are key to explain the impact of entangled residual mappings. We delve deeper and formalize entangled mappings in networks with recurrent nature, networks built by convolutional layers, and networks with attention mechanisms (Transformers).

\noindent\textbf{Introducing Entangled Residual Mappings.} Let $f$ be an arbitrary block of neural network layers, then:
\begin{equation}\label{eq:res}
	R_f(x) := f(x)+x,
\end{equation}
is the residual wrapping of $f$ \citep{he2016deep}. 
The above equation is a special representation of a network combining two pathways: a learnable path $f$ and a fixed connection that multiplies the input with the identity matrix. In particular, Eq~\ref{eq:res} can be rewritten as:$R_f(x) := f(x)+I_n x$,
where $n$ is the dimension of $x$ and $I_n$ the $n\times n$ identity matrix. The Jacobian matrix of this transformation is: $J = \frac{\partial   R_f}{\partial x} \ = \
 \frac{\partial f(x)}{\partial x} \ + \ I_n$.
The motivation behind the fixed pathway is to allow efficient transport of information forward and backward during the training phase \citep{greff2016highway,he2016identity}.

However, the term information in the above motivation is ambiguous, as information propagates in these connections as activations in the forward pass and as errors in the backward flow.
Moreover, the same objective could be achieved by replacing the identity matrix with an arbitrary matrix that shares certain characteristics with the identity matrix.
For instance, literature has studies matrices with absolute eigenvalues 1 \citep{hochreiter2001gradient}, as well as orthogonal and unitary matrices \citep{arjovsky2016unitary}.

We hypothesize that the identity matrix poses a prior on the learned features and that by changing this prior, better architectures can be designed.

To test our hypothesis, we propose an \textit{entangled} variant of a residual block by
\begin{equation}\label{eq:entangled}
	R_f(x) := f(x)+\Gamma x,
\end{equation}
where $\Gamma$ is a constant $n\times n$ matrix that satisfies certain properties.
For instance, in the case of the identity matrix $I_n$, the matrix is symmetric, orthogonal, sparse \citep{liebenwein2021sparse}, and has all eigenvalues at 1.
The Jacobian matrix of the entangled residual connection is given by:
\begin{align}
J = \frac{\partial   R_f}{\partial x} \ = \
\frac{\partial f(x)}{\partial x} \ + \  \Gamma^\text{T} .
\end{align}

We empirically study different families for $\Gamma$ that share only a subset of these properties and evaluate the effects on the performance and learned features.

\noindent \textbf{Eigenvalues, Sparsity and Orthogonality.}
We focus our study on three fundamental properties of the identity matrix; I) all eigenvalues are equal to 1, II) $I_n$ is sparse, and III) $I_n$ is an orthogonal matrix. 

For studying sparsity and the importance of eigenvalues, we define an entanglement matrix that is parameterized by a single scalar variable as
\begin{equation}\label{eq:gamma}
\Gamma := \frac{1}{n}\mathbf{1}_n\gamma + (1-\gamma)I_n,
\end{equation}
where $\mathbf{1}_n$ is the matrix with 1s in each entry and $0\leq \gamma \leq 1$ a hyperparameter.
Sweeping the hyperparameter $\gamma$ results in a spectrum between the standard residual connection ($\gamma=0$) and a fully-entangled connection where the signal $x$ is distributed across all dimensions of $x$ ($\gamma=1$).
Our parametrization of $\Gamma$ has one eigenvalue at 1 and $n-1$ eigenvalues at $(1-\gamma)$ and all singular values strictly less than 1 .

For studying orthogonally entangled residual connections, we set $\Gamma$ to be the orthogonal matrices obtained by the QR-decomposition of randomly generated matrices. Note that naturally all singular values of such orthogonal $\Gamma$ are 1 and all complex eigenvalues have norm 1.

Entangled residual mappings enable us to investigate thoroughly how various deep learning architectures gain benefit from residual connections. Next, we define them in the context of recurrent, convolutional, and attention networks.

\subsection{Entangled Memory RNNs}
\noindent Identity mappings first made their breakthrough in recurrent neural networks. In particular, the LSTM architecture \citep{hochreiter1997long} first propose the state-next-state transition
\begin{equation}\label{eq:lstm}
	c_{t+1} = c_t + i_t\cdot z_t,
\end{equation}
where $c_t$ is the RNN state, $i_t$ and $z_t$ a input gate and activation respectively. Note that the architecture has been later extended by a forget gate that element-wise multiplies with the state $c_t$ in Eq.~\ref{eq:lstm} and improves the learning performance \citep{gers2000learning,greff2016lstm}. Similar to the entangled residual connections, we modify the state-update of the LSTM architecture by replacing the identity mapping by a multiplication with an entanglement matrix, i.e., 
\begin{equation}\label{eq:entangled_lstm}
c_{t+1} = c_t \Gamma + i_t\cdot z_t.
\end{equation}
Note that the LSTM has a second feedback connection through its output gate, thus even if we set $\Gamma$ to all zeros, the architecture is still a recurrent neural network.

\subsection{Entangled Vision and Attention Mappings}

\noindent The hidden units $x$ of contemporary computer vision and attention-based models are not represented by vectors but higher dimensional tensors. For instance, the hidden units $x$ of a 2D convolutional network are rank-3 feature maps, whereas the Transformer architecture \citep{vaswani2017attention} represent its $x$ as rank-2 feature sequences. 

As a result, the entanglement principle given in Eq~\ref{eq:entangled} is not directly applicable to these network architectures. Instead, we propose an entangled residual connection that works on general rank-$k$ tensor objects $x$, by using a $(k-1)$-rank convolution operation, i.e.,
\begin{equation}\label{eq:conv}
	R_f(x) := f(x)+\text{Conv}_{k-1}(\Gamma, x), 
\end{equation}
where $\Gamma$ is a $(k+1)$-ranked kernel tensor. For example, the hidden state of a 2D CNN has rank 3 (width, height, channels) and a corresponding 2D convolution entanglement kernel has rank 4 (weight, height, input channels, output channels). 
Code examples on how to initalize the entanglement kernels for a 2D convolutional neural network is shown in Figure \ref{fig:code}.

\definecolor{LightGray}{gray}{0.9}
\begin{figure}
    \centering
    \begin{minted}[bgcolor=LightGray]{python}
import numpy as np
num_channels = # number of channels
gamme = # Entanglement factor
kernel_size = # Spatial entanglement size
# Orthogonal entanglement
A = np.random.normal(size=(num_channels, num_channels))
orth_kernel, r = np.linalg.qr(A)
orth_kernel = orth_kernel.reshape((1, 1, num_channels, num_channels))
# Spatial entanglment
sp_kernel = np.zeros((kernel_size,kernel_size, num_channels, num_channels))
for i in range(num_channels):
    sp_kernel[:, :, i, i] += gamma / np.square(kernel_size)
    sp_kernel[kernel_size // 2, kernel_size // 2, i, i] +=  1.0 - gamma
# Channel-wise (and channel+spatial) entanglement
ch_kernel = np.ones((kernel_size, kernel_size, num_channels, num_channels)) \
    * gamma / (np.square(kernel_size) * num_channels)
for i in range(num_channels):
    ch_kernel[kernel_size // 2, kernel_size // 2, i, i] +=  1.0 - gamma          
\end{minted}
    \caption{Python code listing for the creation of the different 2D entanglement kernels.}
    \label{fig:code}
\end{figure}

The sparsity and shape of $\Gamma$ determine which dimensions of $x$ get entangled, whereas the non-zero entries define how the dimensions are entangled. Figure \ref{fig:kernels} shows how different dimensions of vision networks can be entangled by various sparse kernels $\Gamma$.
Standard residual connections can be expressed by Eq~(\ref{eq:conv}) as the $\bigotimes_{k-1} 1$ convolution with the identity matrix, e.g. a $1\times 1$ convolution.

\begin{figure}[t]
    \centering
    \includegraphics[width=1.0\textwidth]{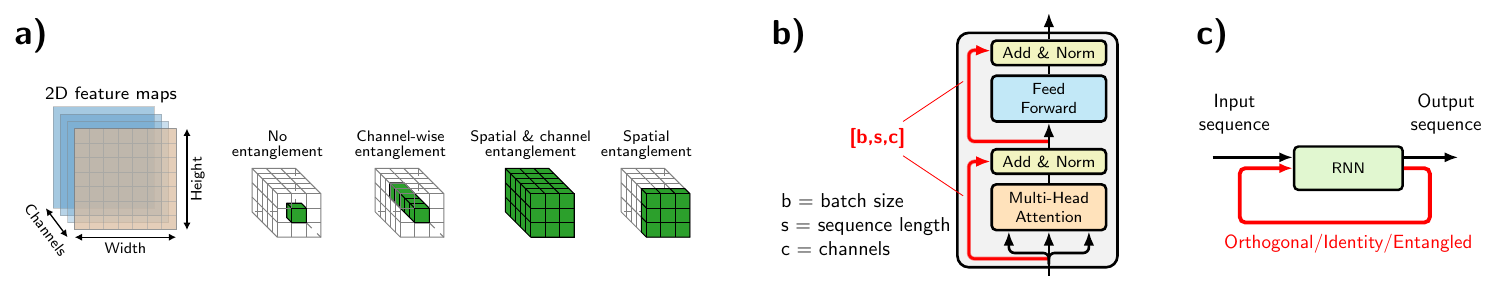}
    \caption{Schematic illustration of various entangled network architectures. \textbf{a)} Different types of entanglement applied to a unit in the first feature channel (orange) of a 2D CNN. \textbf{b)} Entanglement applied to a transformer encoder block. \textbf{c)} Recurrent neural network with orthogonal, identity, or entangled feedback connections. }
    \label{fig:kernels}
\end{figure}

In transformer-encoder blocks \citep{vaswani2017attention} there exist two residual connections. The first one bypasses the multi-head self-attention module, while the second residual pathway bypasses the feed-forward layer that follows the attention module. Each of the two residual layers is followed immediately by a normalization layer, i.e., both modules of the encoder are of the form
\begin{equation}
    \text{LayerNorm}(x+\text{Sublayer}(x)),
\end{equation}
where Sublayer is a multi-head self-attention or a feed-forward network  respectively.
Note that the feed-forward module operates position-wise, i.e., for each sequence item separately \citep{vaswani2017attention}, thus corresponding to a 1D-convolution with a kernel of length 1 and stride 1. As the inputs $x$ of both parts are represented by feature sequences, we apply a 1D-convolutional entanglement as introduced in Equation (\ref{eq:conv}) 
Using the entangled framework, we can comprehensively study the effect of residual connections in transformers in various experimental settings. 

\subsection{Entangled maps preserve iterative refinement}

\noindent Here, we show how entangled residual maps do not affect the iterative refinement of features in residual networks. Iterative inference in ResNets was formalized by \citep{jastrzebski2018residual} as follows: Given a residual block $i$ within a network with $M$ blocks, the transformation of a feature $\textbf{x}_i$, can be formulated by: $\textbf{x}_{i+1} = \textbf{x}_i + f_i(\textbf{x}_i)$.
We can apply a loss function, $\mathcal{L}$, recursively on $M$ \citep{jastrzebski2018residual}:

\begin{equation}
     \mathcal{L}(\textbf{x}_{M}) = \mathcal{L}(\textbf{x}_{M-1} + f_{M-1}(\textbf{x}_{M-1})).
\end{equation}

By computing a first-order Taylor expansion of this loss expression and absorbing higher orders in $\mathcal{O}(.)$, \citep{jastrzebski2018residual} showed that a good approximation can be achieved when $f_{j}$'s magnitude is small in the following expression:

\begin{equation}
    \mathcal{L}(\textbf{x}_{M}) = \mathcal{L}(\textbf{x}_{i}) + \sum^{M-1}_{j=i} f_{j}(\textbf{x}_{j}). \frac{\partial \mathcal{L}(\textbf{x}_{j})}{\textbf{x}_{j}} + \mathcal{O}(f^2_{j}(\textbf{x}_{j})).
    \label{eq:loss_extention}
\end{equation}

\noindent Eq. \ref{eq:loss_extention} implies that when minimizing the loss, $\textbf{x}_i + f_i(\textbf{x}_i)$ moves feature $\textbf{x}_{i}$, with approximately the same amount ($f_{i}(\textbf{x}_{i})$) as that of $\frac{\partial \mathcal{L}(\textbf{x}_{i})}{\textbf{x}_{i}}$ in the opposite direction \citep{jastrzebski2018residual}. This leads to reduction in loss gradually as we go from features $\textbf{x}_{i}$ to $\textbf{x}_{M}$, establishing the \emph{iterative refinement} \citep{greff2016highway,jastrzebski2018residual}. More formally, this means that if a residual block is performing iterative refinement then $x_{i+1} \approx x_i$, which is equivalent to $\frac{\norm{f_{i}(\textbf{x}_{i})}_{2}}{\norm{\textbf{x}_{i}}_{2}}$ being small.

\begin{lemma}\label{lemma:entangle}
Entangled maps with $\gamma$ closer to 0 presented by Eq. \ref{eq:entangled} preserve the iterative feature refinement property of residual blocks that have a small $\frac{\norm{f_{i}(\textbf{x}_{i})}_{2}}{\norm{\textbf{x}_{i}}_{2}}$ ratio. This is while larger values of $\gamma$ may hurt their iterative refinement.
\end{lemma}

\begin{proof}
In a standard residual network some blocks perform iterative feature refinement when the following empirical measure of refinement, $\frac{\norm{f_{i}(\textbf{x}_{i})}_{2}}{\norm{\textbf{x}_{i}}_{2}}$, is small, in which based on Eq. \ref{eq:loss_extention}, $f_{i}(\textbf{x}_{i})$ is the refinement step. Now given the representation of an entangled residual mapping as $f_{i}(\textbf{x}_{i})+\Gamma \textbf{x}_{i}$, with $\norm{\Gamma} = 1$, its refinement step will be given by $\frac{\norm{f_{i}(\textbf{x}_{i})}^{2}_{2}}{\norm{\Gamma \textbf{x}_{i}}^{2}_{2}}$. The following inequality $\norm{\Gamma \textbf{x}_{i}} \leq \norm{\Gamma} \norm{ \textbf{x}_{i}}$ provides a lower bound for the entangled iterative refinement as follows: $\frac{\norm{f_{i}(\textbf{x}_{i})}^{2}_{2}}{\norm{\Gamma}^{2}_{2} \norm{\textbf{x}_{i}}^{2}_{2}} \leq \frac{\norm{f_{i}(\textbf{x}_{i})}^{2}_{2}}{\norm{\Gamma \textbf{x}_{i}}^{2}_{2}}$. We also know based on the definition of $\Gamma$ that it has one eigenvalue at 1 and $n-1$ eigenvalues at $1-\gamma$. Given that $\Gamma$'s minimum eigenvalue is $1-\gamma$ the inequality $ (1-\gamma) \norm{ \textbf{x}_{i}} \leq \norm{\Gamma \textbf{x}_{i}}$ holds. This gives an upper bound for the entangled iterative refinement step as follows:
 $\frac{\norm{f_{i}(\textbf{x}_{i})}^{2}_{2}}{\norm{\Gamma \textbf{x}_{i}}^{2}_{2}} \leq  \frac{\norm{f_{i}(\textbf{x}_{i})}^{2}_{2}}{(1-\gamma)^2 \norm{\textbf{x}_{i}}^{2}_{2}}$. 
 The upper bound shows that for the residual blocks that perform iterative refinement (i.e., has a small $\frac{\norm{f_{i}(\textbf{x}_{i})}_{2}}{\norm{\textbf{x}_{i}}_{2}}$ ratio), entangled maps with $\gamma$ closer to 0 preserves the ratio of the original refinement step. This is while $\gamma$ increases the refinement ratio becomes larger thus the iterative refinement property of the block may not hold.
\end{proof}

The statement of this Lemma has far-reaching implications in practice. In all experiments we observe that a small amount of entanglement ($\gamma = 0.1$) leads to a better generalization while a larger entanglement coefficient $(\gamma = 0.9)$ hurts the performance consistently. Based on Lemma 1, we can conclude that disrupting the iterative refinement of features (choosing a larger $\gamma$, and correspondingly having a larger refinement step) negatively influences residual network's generalization. 

\section{Experiments}
In this section, we conduct a large scale study on how entangled residual connections affect the predictive performance of deep models. In particular, we experiment with advanced vision networks, recurrent networks, and Transformers. 

\begin{table}[t]
\centering
\caption{Top-1 accuracy on ILSVRC 2012 val. set \citep{ILSVRC15} after 90 training epochs. Parenthesis = deviation from baseline model. bold values = best performance.}
\label{tab:imagenet}
\vspace{0mm}
\begin{adjustbox}{width=0.47\columnwidth}
\begin{tabular}{lc}
\toprule
Architecture & Validation accuracy \\
\midrule
ResNet50-v2 baseline & 76.12\% (0.00\%)  \\
\midrule
\multicolumn{2}{l}{Channel-wise entanglement}  \\
$\gamma$ = 0.1 & 76.09\%\  \textcolor{badcolor}{(-0.03\%)} \\
$\gamma$ = 0.5 &  75.63\%\  \textcolor{badcolor}{(-0.49\%)} \\
$\gamma$ = 1.0 & 72.57\%\  \textcolor{badcolor}{(-3.55\%)}\\
\midrule
\multicolumn{2}{l}{Spatial and channel-wise entanglement} \\
$\gamma$ = 0.1 &  76.08\%\  \textcolor{badcolor}{(-0.04\%)} \\
$\gamma$ = 0.5 &  75.69\%\  \textcolor{badcolor}{(-0.43\%)} \\
$\gamma$ = 1.0 & 72.72\%\  \textcolor{badcolor}{(-3.40\%)}\\
\midrule
\multicolumn{2}{l}{Spatial only entanglement} \\
$\gamma$ = 0.1 &  76.31\%\  \textcolor{goodcolor}{(+0.19\%)}\\
$\gamma$ = 0.5 & 75.45\%\  \textcolor{badcolor}{(-0.67\%)}\\
$\gamma$ = 1.0 &  74.65\%\  \textcolor{badcolor}{(-1.47\%)}\\
\midrule
Orthogonal channels &  75.53\%\  \textcolor{badcolor}{(-0.59\%)}\\
\midrule
Entanglement+ &  \textbf{76.69}\% \textcolor{goodcolor}{(+0.57\%)}\\
\midrule
No residual & 72.47\%\ \textcolor{badcolor}{(-3.65\%)} \\
\bottomrule
\end{tabular}
\end{adjustbox}
\end{table}

\textit{Entanglement as a hyperparamter optimization problem.} In all experiments, we tune the hyperparameter once for the non-entangled baseline and then transfer them to the entangled versions of the networks due to computational resource limitations. However, this practice raises the question on whether the baseline's optimal hyperparameters are optimal for the entangled networks as well. 
Moreover, entangled residual connections introduce their own set of hyperparameters such as the interpolation factor $\gamma$. Therefore, we run an additional experiment, where we run a 24-hour search on the entanglement and other hyperparameters. We call these networks \textbf{Entanglement+}. Due to computational limitations, we restrict this experiment to the computer vision benchmarks. Hyperparameters are given in the supplements.

\subsection{Vision Networks}
\noindent \textit{Baselines.} We selected well-studied ResNet50-v2 \citep{he2016identity} and wide-residual \citep{zagoruyko2016wide} for evaluating the properties and influence of entangled maps.

\noindent \textit{Datasets.} We benchmark networks on CIFAR-10, CIFAR-100 \citep{krizhevsky2009learning}, and ILSVRC-2012 (referred to as ImageNet) \citep{ILSVRC15}.

\noindent \textit{Experimental Setup ImageNet.} ImageNet represents an image classification task composed of 1.28 million publicly available training and 50,000 validation images corresponding to 1000 possible object categories.
Before training, we reprocess each image by normalizing each color-channel to zero mean and unit standard deviation over the entire training dataset. We train the entangled ResNet50-v2 models on the ImageNet dataset for 90 epochs with a batch size of 512. We linearly warm-up the learning rate from $10^{-3}$ to 0.5 in the first 5 epochs, followed by a linear-cosine decay \citep{bello2017neural} to $10^{-4}$ in the remaining 85 epochs. 

\begin{table}[t]
\centering
	\caption{Classification acc. on CIFAR-10 and CIFAR-100 test set \citep{krizhevsky2009learning}. Best values in bold. n=3}
	\label{tab:cifar}
\begin{adjustbox}{width=0.57\columnwidth}
			\begin{tabular}{lcc}
				\toprule
				Architecture & CIFAR-10 & CIFAR-100 \\
				\midrule
				Baseline WRN-28-10 & 96.37\% $\pm$ 0.07 & 80.56\% $\pm$ 0.16 \\
				\midrule
				\multicolumn{3}{l}{Channel-wise entanglement}  \\
				$\gamma$ = 0.1 &  96.32\% $\pm$ 0.14  & 80.59\% $\pm$ 0.28  \\
				$\gamma$ = 0.5 &  96.37\% $\pm$ 0.05  & 78.97\% $\pm$ 0.44 \\
				$\gamma$ = 1.0 &  95.00\% $\pm$ 0.75  & 75.89\% $\pm$ 0.87 \\
				\midrule
				\multicolumn{3}{l}{Spatial and channel-wise entanglement} \\
				$\gamma$ = 0.1 &  96.43\% $\pm$ 0.05  & 80.78\% $\pm$ 0.11  \\
				$\gamma$ = 0.5 &  96.33\% $\pm$ 0.13  & 79.08\% $\pm$ 0.17  \\
				$\gamma$ = 1.0 &  95.86\% $\pm$ 0.10  & 77.39\% $\pm$ 0.11 \\
				\midrule
				\multicolumn{3}{l}{Spatial only entanglement} \\
				$\gamma$ = 0.1 & 96.49\% $\pm$ 0.10 &  80.82\% $\pm$ 0.16  \\
				$\gamma$ = 0.5 & 96.41\% $\pm$ 0.13  & 79.70\% $\pm$ 0.12  \\
				$\gamma$ = 1.0 & 95.96\% $\pm$ 0.08  & 79.14\% $\pm$ 0.26 \\
				\midrule
                Orthogonal channels & 96.27\% $\pm$ 0.14 & 80.33\% $\pm$ 0.15 \\
				\midrule
				No residual &  95.73\% $\pm$ 0.23  & 77.55\% $\pm$ 0.17  \\\midrule
				Entanglement+ & \textbf{96.72\% $\pm$ 0.14} & \textbf{81.27\% $\pm$ 0.26} \\
				\midrule\bottomrule
				Kernel orth. \citep{xie2017need} & - & 79.3\% $\pm$ N/A  \\
				OCNN \citep{Wang_2020_CVPR} & - & 80.1\% $\pm$ N/A\\
				\bottomrule
			\end{tabular}
		\end{adjustbox}
\end{table}
For scaling the SGD updates, we use a Nesterov momentum of 0.875 \citep{nesterov27method}. We apply standard random cropping of $224\times 224$ segments, random aspect ratio distortions (3/4 to 4/3), and random horizontal flips for data augmentation. We add a 1/32768 L2 weight decay to the Conv layers' kernels and perform label smoothing with a factor of 0.1 \citep{szegedy2016rethinking}.

\noindent \textit{Experimental Setup CIFAR10 and 100.} Each sample is a 32-by-32 RGB image corresponding to one out of 10 (CIFAR-10) or 100 (CIFAR-100) categories.
Before training, we reprocess each image by normalizing each color-channel to zero mean and unit standard deviation over the entire training dataset. 
We train our entangled WRN-28-10 models on the CIFAR-10 and CIFAR-100 dataset for 200 epochs with a batch size 256. 
We use a warm-up epoch with a learning rate of 0.005.
For the remaining epochs, we use a learning rate of 0.1, which we decay by a factor of 0.2 every 60 epochs. 
For scaling the gradient steps, we use a Nesterov-momentum of 0.9 \citep{nesterov27method}. We apply cutout as a data augmentation strategy \citep{devries2017improved} and add an L2 weight decay factor of 0.0005 to the Conv layers' kernels. We repeat each experiment 3 times and report the mean and std.

\noindent \textbf{Discussing the Entangled Vision Networks}
The results are shown in Table \ref{tab:imagenet} and Table \ref{tab:cifar} for ImageNet and CIFAR, respectively. We make three key observations:
Having full ($\gamma=1$) channel-wise or channel+spatial entangled residual connections is as bad or even worse than having no skip connections. We hypothesize that a channel-wise entanglement reduces the learned features' diversity as the residual blocks are biased toward combining different feature channels. To further test our hypothesis by performing a feature visualization in the following.  
Small spatial entanglement ($\gamma=0.1$) consistently improves the performance across all benchmarks. Networks with orthogonal maps perform worse than identity maps. We further elaborate on this below.

We perform a qualitative analysis to investigate how an entanglement affects the learned features. To do this, we leverage existing methods developed for feature visualization and attribution mapping. Note that due to the training process's stochastic nature, the learned features of networks are not directly comparable. To overcome this limitation and have a direct comparison, we load the baseline model weights but use the entangled architectures (with $\gamma=0.1$). Consequently, we visualize how entanglement changes the features learned by the standard model. Ideally, as we use a small $\gamma$, we would want the entangled mapping to preserve the base network's features to a large extend.


\begin{figure}[t]
\centering
\includegraphics[width=0.6\textwidth]{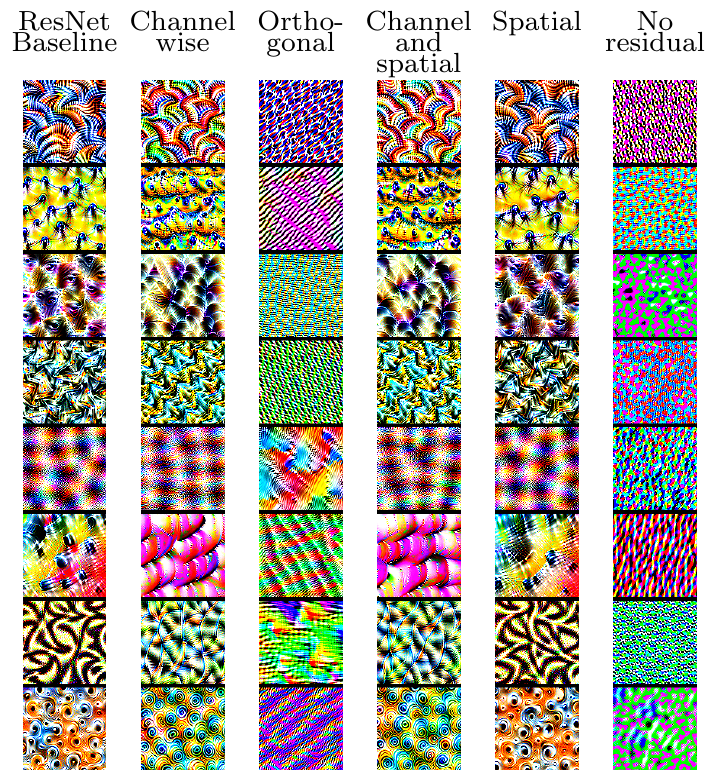}
\caption{Feature visualization of the first eight channels after 6-th residual block. Due to entanglement, the channels of the channel-wise entangled network have learned similar features, explaining partially its poor accuracy.
}
\label{fig:features}
\end{figure}

\noindent \textbf{Feature visualization}
Fig. \ref{fig:features} visualizes the first eight feature channels after the 6-th residual block by maximizing the mean activation over the entire channel \citep{erhan2009visualizing}. We use 1000 iterations of Adam \citep{kingma2014adam} (step-size 0.05) for the optimization and apply a Fourier parametrization \citep{olah2017feature}, jittering \citep{mordvintsev2015inceptionism}, and color-space decorrelation \citep{olah2017feature} to regularize the visualization. 

We observe that removing residual connections entirely results in a poor filter representation, which shows that res-nets encode their features completely different than pure sequential models. We make the same conclusion for orthogonally entangled network. 

\begin{figure}[t]
\centering
\includegraphics[width=0.5\textwidth]{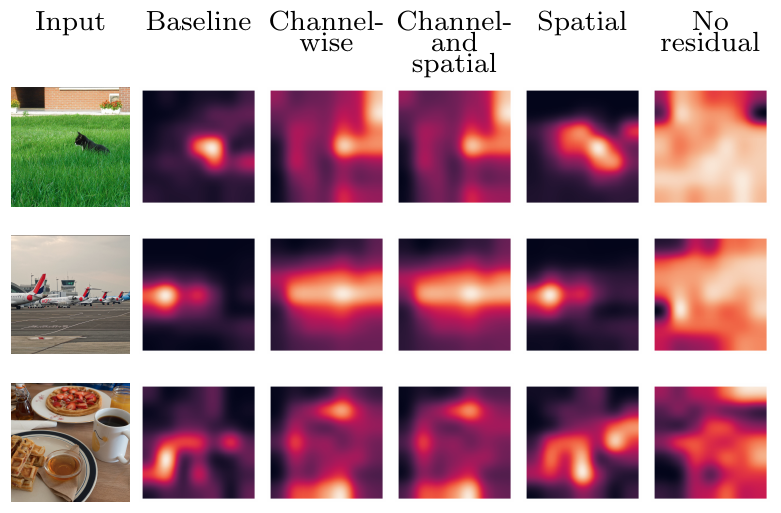}
\caption{Saliency maps of the entangled ResNets trained on ImageNet. 
}
\label{fig:saliency}
\end{figure}
Spatial entanglement best preserves the learned features' patterns among entangled variants. It only modifies the pattern's location and rotation. Channel-wise entanglement alters the representation of some features compared to the baseline, which suggests that it has a much stronger effect on the feature encoding compared to a spatial mode.


\noindent \textbf{Saliency maps.}
Figure \ref{fig:saliency} visualizes the attention (saliency) maps of networks via Grad-CAM \citep{selvaraju2017grad}. Like in our feature visualization, we use entangled networks, with $\gamma=0.1$, and load the baseline weights to compare more directly. Our observations on the attention maps are aligned with what we learned from feature visualization. 
We notice that the network's attention is entirely scattered if the residual connections are not present.
We also notice that the saliency maps with spatial entanglement preserve the correlations required for reasoning in these examples, demonstrating that it learns similar weight representations as the baseline network. On the other hand, channel-wise entanglement reduces the clarity of the attention maps, which suggests strong alternation of networks' weight representation.



\begin{table}[t]
\centering
\caption{Top-1 accuracy on the ILSVRC 2012 validation set \citep{ILSVRC15} after 300 training epochs. Parenthesis = deviation from baseline model. bold = best performance.}
\label{tab:vit}
\begin{adjustbox}{width=0.45\columnwidth}
\begin{tabular}{lc}
\toprule
Architecture & Validation accuracy \\
\midrule
Vision transformer (ViT-S) & 72.75\%  \\
\midrule
\multicolumn{2}{l}{Channel-wise entanglement}  \\
$\gamma$ = 0.1 & \textbf{73.14}\% \\
$\gamma$ = 1.0 & 0.1\% \\
\midrule
\multicolumn{2}{l}{Spatial and channel-wise entanglement} \\
$\gamma$ = 0.1 &  \textbf{73.39}\%  \\
$\gamma$ = 1.0 & 0.1\% \\
\midrule
\multicolumn{2}{l}{Spatial only entanglement} \\
$\gamma$ = 0.1 & 71.86\% \\
$\gamma$ = 1.0 & 57.29\% \\
\midrule
Orthogonal channels & 72.49\% \\
\bottomrule
\end{tabular}
\end{adjustbox}
\end{table}
\subsection{Vision Transformers}
Here, we test whether entangled mappings have the same effect on vision transformers \citep{dosovitskiy2020image}. A vision transformer (ViT) tiles the input images into a sequence of non-overlapping patches, followed by a linear projection applied to each patch. The resulting feature sequence is then processed by a transformer architecture \citep{vaswani2017attention}.  

We re-use the ImageNet training pipeline of our first experiment for training entangled versions of the vision transformer. Due to memory and compute limitations, train the smallest architecture (ViT-S) consisting of 12 layers, each with a feature dimension of 384, 6 attention heads, a feed-forward hidden size of 1536, input image size of 224-by-224, and a patch size of 32-by-32 pixels. For the optimization we employ AdamW \citep{loshchilov2017decoupled} and train for 300 epochs. As ViTs are known to require more training data than CNNs, we do not evaluate them on the CIFAR datasets.

The results in Table \ref{tab:vit} show that, conversely to the CNN, a spatial entanglement degrades the performance of a ViT. However, a channel-wise and channel+spatial entanglement improves accuracy. Our results suggest that transformer architecture makes use of the residual connections differently than residual CNNs. Similar observations have been reported before \citep{raghu2021vision}.

\subsection{Entanglement in Recurrent Networks} 
\noindent In this section, we evaluate entangled mappings in the context of recurrent networks. In RNNs, the mapping is applied to the network's hidden-to-hidden update transition. Consequently, the entanglement affects how an RNN's memory changes over time. As $\gamma$ in Eq. \ref{eq:gamma} increases, the memory content of a single unit distributes to other units over time. We hypothesize that the entanglement reduces the RNN's memory capacity but might improve learning to store and retrieve. On the other hand, an orthogonal entanglement might have little effect on the capacity but negatively influences the network's ability to store and load the content. We test this hypothesis in a series of four sequence classification tasks with different memory-horizons. 

\noindent \textit{RNNs Task Setup and Datasets.} The first evaluation concerns the sequential MNIST task \citep{lecun1998gradient}, where a 28-by-28 gray-scale image representing a digit from 0 to 9 is fed into the RNN in a pixel-after-pixel fashion, i.e., as a sequence of 784 steps. The objective is to classify which digit the sequence represents. The second task is a modification of the first task, where the pixels' order is permuted, i.e., altering the temporal correlation of the features.

\begin{table}[t]
\centering
\small
\caption{Sequence classification tasks for entangled LSTM models. Mean and standard deviation ($n=5$). Best values in bold.}
\begin{adjustbox}{width=1\columnwidth}
\begin{tabular}{lcccc}
    \toprule
    Entanglement  & Sequential MNIST & Permutated  & IMDB sentiment & IMDB sentiment \\
    factor & & Sequential MNIST & (unidirectional RNN) & (bidirectional RNN) \\
    \midrule
    $\gamma$ = 0 & \textbf{98.79}\% $\pm$ 0.04      & {93.31}\% $\pm$ 0.18 & \textbf{86.33}\% $\pm$ 0.21 & \textbf{86.42}\% $\pm$ 0.61 \\
    $\gamma$ = 0.1 &  82.87\% $\pm$ 5.29 & 92.16\% $\pm$ 0.17 & 84.74\% $\pm$ 0.61  &  83.44\% $\pm$ 2.29 \\
    $\gamma$ = 0.25 &  46.80\% $\pm$ 7.95 & 90.74\% $\pm$ 0.24 & 82.95\% $\pm$ 1.44  &  83.31\% $\pm$ 0.75 \\
    $\gamma$ = 0.5 &  37.06\% $\pm$ 3.35 & 87.06\% $\pm$ 2.89 & 82.89\% $\pm$ 1.97   & 81.44\% $\pm$ 1.34  \\
    $\gamma$ = 0.75 &  34.00\% $\pm$ 1.36 & 87.11\% $\pm$ 0.34 &  81.59\% $\pm$ 1.60 & 79.65\% $\pm$ 2.82  \\
    $\gamma$ = 1.0 &  45.39\% $\pm$ 3.01 & 85.49\% $\pm$ 1.24 & 79.23\% $\pm$ 0.78  & 76.85\% $\pm$ 2.58  \\
    \midrule
    Orthogonal & 98.01\% $\pm$ 0.17 & \textbf{95.74}\% $\pm$ 0.10 & 75.67\% $\pm$ 2.85 & 74.06\% $\pm$ 3.63\\
    \midrule
    No residual &  26.47\% $\pm$ 7.67 & 85.13\% $\pm$ 0.48 &  78.12\% $\pm$ 0.84 &  73.84\% $\pm$ 1.71 \\
    \bottomrule
\end{tabular}
\end{adjustbox}
\label{tab:lstm}
\end{table}

For our second evaluation, we use the IMDB sentiment analysis dataset benchmarked \citep{maas2011learning}. The dataset consists of 25,000 training and 25,000 test sentences, each corresponding to either positive or negative sentiment. The sentences are tokenized with a vocabulary consisting of 20,000 most frequent words.

In the third experiment, we train a stack of two entangled LSTM layers that process the text sequence in the reading direction, i.e., from left-to-right. The data is truncated to 200 tokens. Sequences that are shorter are right-padded to 200 tokens with empty tokens. As the sequences are shorter than in our first two tasks , long-term memory formation is less of an issue in this dataset. In the fourth experiment, we let both LSTM layers traverse over the sequences from left-to-right and right-to-left and concatenate both traverses' outputs before feeding them into the next layer. Both LSTMs are composed of 256 units. We use Adam \citep{kingma2014adam} with a learning rate of 0.0005.

\noindent \textbf{Discussing Results for Entangled Recurrent Networks.}
The results are shown in Table \ref{tab:lstm}. As expected, entangling an LSTM network's memory element toward a constant-ones matrix consistently decreases its long-term learning performance. However, counterintuitively, this reduction is also present in sentiment datasets, which arguably have less long-term components.
We conclude that a memory entanglement of an RNN is complex and affects other properties of the network compared to feed-forward models.

Orthogonal LSTM outperforms the standard LSTM on the permuted sequential MNIST dataset but performs worse than having no memory at all in both NLP experiments. RNNs that use a unitary or orthogonal hidden-to-hidden transition matrix \citep{arjovsky2016unitary} are able to learn long-term dependencies. For instance, in the task permuted and sequential MNIST tasks with a memory horizon of 784, the orthogonal variant of the LSTM achieves excellent performance.

Although orthogonal RNNs can store information for a very long time, the representation of this information changes after each RNN step due to the non-identity hidden-to-hidden transition. As a result, at the sequence end, the stored information looks different depending on when it was stored.
Consequently, orthogonal/unitary RNNs realize an inductive bias of making their memory content sensitive to when and for how long some information was stored. Conversely, they struggle with transferring representations learned from one time step to another time step. In some application domains such as NLP tasks, such time-domain knowledge transfer is crucial and explains the poor generalization of the orthogonal LSTM on these datasets.
On the other hand, on the permuted sequential MNIST, the knowledge of the exact temporal location is very important as the permutation has destroyed the natural temporal structure.

\section{Conclusion}
\noindent In this paper, we introduced entangled residual mappings as a framework to generalize residual connections beyond identity mappings across various deep learning architectures. Within this framework, we studied the role of eigenvalues, sparsity, and orthogonality of residual mappings. We performed a comprehensive experimental evaluation of residual mappings' role in different architectures ranging from advanced ResNets and recurrent networks to Transformers. We showed that incorporating structural biases into residual connections can improve representation learning.

\begin{wraptable}[8]{r}{0.5\textwidth}
\vspace{0mm}
    \centering
    \caption{Entangles nets with different depth trained on ImageNet.}
   \begin{adjustbox}{width=0.5\columnwidth}
    \begin{tabular}{l|cccc}\toprule
         Entanglement &  ResNet34 & ResNet50 & ResNet101 & ResNet152\\\midrule
         Identity & \textbf{72.15}\% & 76.12\% & 79.00\% & 80.47\%\\
         Orthogonal & 71.36\% & 75.53\% & 78.25\% & 79.54\% \\
         Spatial ($\gamma=0.1$) & 70.96\% & \textbf{76.31}\% & \textbf{80.19}\% & \textbf{81.67}\% \\\bottomrule
    \end{tabular}
    \end{adjustbox}
    \label{tab:depth}
\end{wraptable}
\noindent \textbf{Entanglement gets more prominent with increasing depth}
Here, we evaluate how the depth of a network modulates the effects of entangled residual connections. In particular, we vary the depth of a residual trained on ImageNet entangled by either an identity, orthogonal, or spatial ($\gamma=0.1$) mapping for 90 epochs. The results in Table \ref{tab:depth} show that orthogonal entanglement hurts performance independently of the network's depth, which is consistent with our earlier observations. Moreover, we observed that spatial entanglement slightly degrades the performance on the shallow network, but improves accuracy with increasing depth of the network.

%

\noindent \textbf{Limitations.} Performance gain by entangled mappings although improved compared to identity mappings in some experiments is marginal. This is due to the fact that entanglement does not aim to disrupt the iterative feature refinement property of residual networks with identity mappings. Instead, the framework suggests theoretically and experimentally that iterative feature refinement should not be disrupted, if so, the performance degrades as seen with cases where the entanglement coefficient was set too high.

\section*{Acknowledgments} 
M.L. and T.A.H. are supported in part by the Austrian Science Fund (FWF) under grant Z211-N23 (Wittgenstein Award). R.H. and D.R. are partially supported by Boeing and MIT. Z.B. is supported by the Doctoral College Resilient Embedded Systems, which is run jointly by the TU Wien's Faculty of Informatics and the UAS Technikum Wien. R.G. is partially supported by the Horizon 2020 Era-Permed project Persorad. This research was partially sponsored by the United States Air Force Research Laboratory and the United States Air Force Artificial Intelligence Accelerator and was accomplished under Cooperative Agreement Number FA8750-19-2-1000. The views and conclusions contained in this document are those of the authors and should not be interpreted as representing the official policies, either expressed or implied, of the United States Air Force or the U.S. Government. The U.S. Government is authorized to reproduce and distribute reprints for Government purposes notwithstanding any copyright notation herein. This work was further supported by The Boeing Company and the Office of Naval Research (ONR) Grant N00014-18-1-2830.




\end{document}